\documentclass[runningheads]{llncs}
\usepackage{color}
\usepackage{graphicx}
\graphicspath{{images/}}
\usepackage{amsfonts}
\usepackage{verbatim} 

\usepackage{amsmath,amsfonts,bm}

\def\eqref#1{equation~\ref{#1}}

\def\1{\bm{1}}

\def\vb{{\bm{b}}}

\def\vf{{\bm{f}}}

\def\vh{{\bm{h}}}

\def\vx{{\bm{x}}}
\def\vy{{\bm{y}}}

\def\evb{{b}}

\def\evg{{g}}
\def\evh{{h}}

\def\evx{{x}}
\def\evy{{y}}
\def\evz{{z}}

\def\mW{{\bm{W}}}

\DeclareMathAlphabet{\mathsfit}{\encodingdefault}{\sfdefault}{m}{sl}
\SetMathAlphabet{\mathsfit}{bold}{\encodingdefault}{\sfdefault}{bx}{n}

\def\sD{{\mathbb{D}}}

\def\sL{{\mathbb{L}}}

\def\sN{{\mathbb{N}}}

\def\sX{{\mathbb{X}}}

\newcommand{\dnn}{\textsc{DNN} }
\newcommand{\dnns}{\textsc{DNN}s}
\newcommand{\dnnone}{$\mathcal{N}_1$}
\newcommand{\dnntwo}{$\mathcal{N}_2$}

\newcommand{\fone}{\vf_1}
\newcommand{\ftwo}{\vf_2}

\usepackage{multirow}

\usepackage{algorithm}
\usepackage{algpseudocode}

\begin{document}
\title{Lossless Compression of Deep Neural Networks}
%
%
\author{Thiago Serra\inst{1}
\and
Abhinav Kumar\inst{2}
\and
Srikumar Ramalingam\inst{2}
} 
\authorrunning{T. Serra et al.}
%
\institute{Bucknell University, USA 
\\
\email{thiago.serra@bucknell.edu}
\and
The University of Utah, USA
\\
\email{abhinav.kumar@utah.edu,srikumar@cs.utah.edu}
}
\maketitle              

\begin{abstract}
Deep neural networks have been successful in many predictive modeling tasks, such as image and language recognition,  
where large neural networks are often used to obtain good accuracy. 
Consequently, 
it is challenging to deploy these networks under 
limited computational resources, such as in mobile devices. 
In this work, we introduce an algorithm that 
removes units and layers of a neural network 
while not changing the output that 
is produced, 
which thus implies a lossless compression. 
This algorithm, which we denote as \texttt{LEO}~(Lossless Expressiveness Optimization), relies on Mixed-Integer Linear Programming~(MILP) to identify Rectified Linear Units~(ReLUs) with  linear behavior over the input domain. 
By using $\ell_1$ regularization to induce such behavior, 
we can benefit from training over a larger architecture than we would later use in the environment where the trained neural network is deployed. 

\end{abstract}

\begin{table}
\caption{Compression of 2-hidden-layer rectifier networks trained on MNIST. Each line summarizes tests on 31 networks. Depending on how the network is trained, the higher incidence of stable units allows for more compression while preserving the trained network accuracy. For example, training with $\ell_1$ regularization induces such stability and then inactive units can be  removed. Interestingly, the small amount of regularization that improves accuracy during training also helps compressing the network later.}
\label{tab1}
\centering
\begin{tabular}{@{\extracolsep{6pt}}cccccc}
&&& \multicolumn{2}{c}{Units removed} &  Network \\
\cline{4-5}
Layer width & $\ell_1$ weight & Accuracy~(\%) &  $1^{\text{st}}$ layer & $2^{\text{nd}}$ layer &  compression~(\%) \\
\hline
25 & 0.001 & 95.76 $\pm$ 0.05 & 5.7 $\pm$ 0.3 & 5.1 $\pm$ 0.3 & 22 $\pm$ 1 \\ 
25 & 0.0002 & 97.24 $\pm$ 0.02 & 1.2 $\pm$ 0.1 & 3.0 $\pm$ 0.4 & 8.3 $\pm$ 0.7 \\ 
25 & 0 & 96.68 $\pm$ 0.03 & 0 $\pm$ 0 & 0 $\pm$ 0 & 0 $\pm$ 0 \\ 
\hline
50 & 0.001 & 96.05 $\pm$ 0.04 & 16.9 $\pm$ 0.6 & 12.5 $\pm$ 0.6 & 29.4 $\pm$ 0.7 \\ 
50 & 0.0002 & 97.81 $\pm$ 0.02 & 7.6 $\pm$ 0.4 & 7.5 $\pm$ 0.5 & 15.1 $\pm$ 0.6 \\ 
50 & 0 & 97.62 $\pm$ 0.02 & 0 $\pm$ 0 & 0 $\pm$ 0 & 0 $\pm$ 0 \\ 
\hline
100 & 0.0005 & 97.14 $\pm$ 0.02 & 36.7 $\pm$ 0.7 & 24.9 $\pm$ 0.6 & 30.8 $\pm$ 0.5 \\ 
100 & 0.0001 & 98.14 $\pm$ 0.01 & 18.6 $\pm$ 0.5 & 11.1 $\pm$ 0.7 & 14.9 $\pm$ 0.4 \\ 
100 & 0 & 98.00 $\pm$ 0.01 & 0 $\pm$ 0 & 0 $\pm$ 0 & 0 $\pm$ 0 \\ 
\end{tabular}
\end{table}

\section{Introduction}
    Deep Neural Networks (\dnns) have achieved unprecedented success in many domains of predictive modeling, such as computer vision~\cite{Krizhevsky2012,Ciresan2012,Goodfellow2013,Szegedy2015,He2016DeepRL}, natural language processing~\cite{sutskever2014sequence}, and speech~\cite{Hinton2012}. While complex architectures are usually behind such feats, it is not fully known if these results depend on such \dnns~being as wide or as deep as they currently are for some applications. 
    
    In this paper, we are interested in the compression of \dnns, especially to reduce their size and depth. 
    More generally, that relates to the following question of wide interest about neural networks: given a neural network \dnnone, can we find an \emph{equivalent} neural network \dnntwo~with a different architecture? Since a trained \dnn corresponds to a function mapping its inputs to outputs, we can formalize the equivalence among neural networks as follows~\cite{BinarizedProperties}:
    \begin{definition}[Equivalence]
        Two deep neural networks \dnnone~and \dnntwo~with associated functions $\fone : \mathbb{R}^{n_0}\rightarrow \mathbb{R}^{m}$ and $\ftwo : \mathbb{R}^{n_0}\rightarrow \mathbb{R}^{m}$, respectively, are equivalent if $\fone(\vx) = \ftwo(\vx) ~\forall~\vx \in \mathbb{R}^{n_0}$.
        \label{def:equivalent_networks_1}
    \end{definition}
    In other words, our goal is to start from a trained neural network and identify neural networks with fewer layers or smaller layer widths that would produce the exact same outputs. Since the typical input for certain applications is bounded along each dimension, such as $\vx \in [0,1]^{n_0}$ for the MNIST dataset~\cite{LeCun1998}, we can consider a broader family of neural networks that would be regarded as equivalent in practice. We formalize that idea with the concept of local equivalence:
    \begin{definition}[Local Equivalence]
        Two deep neural networks \dnnone~and \dnntwo~with associated functions $\fone : \mathbb{R}^{n_0}\rightarrow \mathbb{R}^{m}$ and $\ftwo : \mathbb{R}^{n_0}\rightarrow \mathbb{R}^{m}$, respectively, are local equivalent with respect to a domain $\sD \subseteq \mathbb{R}^{n_0}$ if $\fone(\vx) = \ftwo(\vx) ~\forall~\vx \in \sD$.
        \label{def:equivalent_networks_2}
    \end{definition}

    For a given application, local equivalence with respect to the domain of possible inputs suffices to guarantee that two networks have the same accuracy in any test. Hence, finding a smaller network that is local equivalent to the original network implies a compression in which there is no loss. In this paper, we show that simple operations such as removing or merging units and folding layers of a \dnn can yield such lossless compression under certain conditions. We denote as \emph{folding} the removal of a layer by directly connecting the adjacent layers, which is accompanied by adjusting the weights and biases of those layers accordingly. 
\section{Background}
    We study feedforward \dnns~with Rectified Linear Unit~(ReLU) activations~\cite{OriginReLU}, which are comparatively simpler than other types of activations. Nevertheless, ReLUs are currently the type of unit that is most commonly used~\cite{CurrentDNN}, which is in part due to landmark results showing their  competitive performance~\cite{nair2010rectified,ReLUGood2}.  
    
    Every network has input $\vx = [\evx_1 ~ \evx_2 ~ \dots ~ \evx_{n_0}]^T$ from a bounded domain $\sX$ and corresponding output $\vy = [\evy_1 ~ \evy_2 ~ \dots ~ \evy_m]^T$, and each hidden layer $l \in \sL = \{1,2,\dots,L\}$ has output $\vh^l = [\evh_1^l ~ \evh_2^l \dots \evh_{n_l}^l]^T$ from ReLUs indexed by $i \in \sN_l = \{1, 2, \ldots, n_l\}$. 
    Let $\mW^l$ be the $n_l \times n_{l-1}$ matrix where each row corresponds to the weights of a neuron of layer $l$, and  let $\vb^l$ be vector of biases associated with the units in layer $l$. With $\vh^0$ for $\vx$ and $\vh^{L+1}$ for $\vy$,  the output of each unit $i$ in layer $l$ consists of an affine function $\evg_i^l = \mW_{i}^l \vh^{l-1} + \vb_i^l$ followed by the ReLU activation $\evh_i^l = \max\{0, \evg_i^l\}$. The unit $i$ in layer $l$ is denoted \emph{active} when $\vh_i^l > 0$ and \emph{inactive} otherwise. \dnns~consisting solely of ReLUs are denoted \emph{rectifier} networks, and their associated functions are always piecewise linear~\cite{arora2018understanding}.

    \subsection{Mixed-Integer Linear Programming}
        Our work is primarily based on the fast growing literature on applications of Mixed-Integer Linear Programming~(MILP) to deep learning.  
        MILP can be used to map inputs to outputs of each ReLU and consequently of  rectifier networks. Such formulations have been used to produce the image~\cite{LomuscioMIP,DuttaMIP} and estimate the number of pieces~\cite{serra2018bounding,serra2020empirical} of the piecewise linear function associated with the network, generate adversarial perturbations to test the network robustness~\cite{Cheng2017,FischettiMIP,tjeng2017evaluating,singh2018robustness,xiao2018training,anderson2019strong}, and implement controllers based on \dnn  models~\cite{Planning,XADD}.
        
        For each unit $i$ in layer $l$, we can map $\vh^{l-1}$ to $g_i^l$ and $h_i^l$ with a formulation that also includes a binary variable $\evz_i$ denoting if the unit is active or not, a variable $\bar{\evh}^l_i$ denoting the output of a complementary fictitious unit $\bar{\evh}^l_i = max\left\{0, -\evg_i^l \right\}$, and 
        constants $H_i^l$ and $\bar{H}_i^l$ that are positive and as large as $\evh_i^l$ and $\bar{\evh}_i^l$ can be: 
        \begin{align}
            \mW_i^l \vh^{l-1} + \evb_i^l = \evg_i^l \label{eq:mip_unit_begin} \\
            \evg_i^l = \evh_i^l - \bar{\evh}_i^l  \label{eq:mip_after_Wb_begin} \\
            \evh_i^l \leq H_i^l \evz_i^l \\ 
            \bar{\evh}_i^l \leq \bar{H}_i^l (1-\evz_i^l) \\
            \evh_i^l \geq 0 \\
            \bar{\evh}_i^l \geq 0 \\
            \evz_i^l \in \{0, 1\} \label{eq:mip_unit_end}
        \end{align}
        This formulation can be strengthened by using the smallest possible values for $H_i^l$ and $\bar{H}_i^l$~\cite{FischettiMIP,tjeng2017evaluating} and valid inequalities to avoid fractional values of $z_i^l$~\cite{anderson2019strong,serra2020empirical}. 
        
        The largest possible value of $\evg_i^l$, which we denote $\mathcal{G}_i^l$, can be obtained as 
        \begin{align}
            \mathcal{G}_i^{l'} = & \max ~ & \mW_i^{l'} \vh^{{l'}-1} + \evb_i^{l'} \label{eq:obj} \\
            &\text{s.t.} ~ &  \text{(\ref{eq:mip_unit_begin})--(\ref{eq:mip_unit_end})}  & \qquad  \forall l \in \{1, \ldots, l'-1\}, ~ i \in \sN_l \\
            && x \in \sX \label{eq:last_mi}
        \end{align}
        If $\mathcal{G}_i^l > 0$, then $\mathcal{G}_i^l$ is also the largest value of $\evh_i^l$ and it can be used for constant $H_i^l$. Otherwise, $\evh_i^l = 0$ for any input $x \in \sX$. We can also minimize $\mW_i^l \vh^{l-1} + \evb_i^l$ to obtain $\overline{\mathcal{G}}_i^l$, the smallest possible value of $\evg_i^l$, and use $-\overline{\mathcal{G}}_i^l$ for constant $\bar{H}_i^l$ if $\overline{\mathcal{G}}_i^l < 0$; whereas $\overline{\mathcal{G}}_i^l > 0$ implies that $\evh_i^l > 0$ for any input $x \in \sX$. By solving these formulations from the first to the last layer, we have the tightest values for $H_i^l$ and $\bar{H}_i^l$ for $l \in \{1, \ldots l'-1\}$ when we reach layer $l'$. 
        Units with only zero or positive outputs were first identified using MILP in~\cite{tjeng2017evaluating}, where they are denoted as \emph{stably inactive} and \emph{stably active}, and their incidence was induced with $\ell_1$ regularization. That was later exploited to accelerate the verification of robustness by making the corresponding MILP formulation easier to solve~\cite{xiao2018training}.
    
        In this paper, we use the stability of units to either remove or merge them while preserving the outputs produced by the DNN.
        %
        The same idea could be extended 
        to other architectures with MILP mappings, such as Binarized Neural Networks (BNNs)~\cite{courbariaux2016binarized,BinarizedProperties}. MILP has been used in BNNs for adversarial testing~\cite{BinarizedAttack} and along with Constraint Programming~(CP) for training~\cite{BinarizedTraining}. BNNs have characteristics that also make them suitable 
        under limited  
        resources~\cite{BinarizedProperties}.

    \subsection{Related Work}
        Our work relates to the literature on neural network compression, and more specifically to methods that simplify a trained DNN. 
        Such literature includes 
        low-rank decomposition~\cite{denton2014exploiting,jaderberg2014speeding,zhang2015efficient,wang2018wide,peng2018extreme,dubey2018coreset}, quantization~\cite{rastegari2016xnor,courbariaux2016binarized,wu2016quantized,tung2018clip}, architecture design~\cite{Szegedy2015,iandola2016squeezenet,howard2017mobilenets,huang2017densely,tang2019towards}, non-structured pruning~\cite{lin2018synaptic}, structured pruning \cite{han2015learning,li2016pruning,molchanov2016pruning,han2016dsd,luo2017thinet,aghasi2017net,yu2018nisp}, sparse learning~\cite{liu2015sparse,zhou2016less,alvarez2016learning,wen2016learning}, automatic discarding of layers in ResNets~\cite{Veit2017,yu2018learning,Herrmann2018}, variational methods~\cite{zhao2019variational}, and the recent Lottery Ticket Hypothesis~\cite{LotteryTicket} by which training only certain subnetworks in the DNN --- the \emph{lottery tickets} --- 
        might be good enough. 
        However, network compression is often achieved with side effects to the function associated with the DNN. 
        
        In contrast to many lossy pruning methods that typically focus on removing unimportant neurons and connections, our approach focuses on developing lossless transformations that exactly preserve the expressiveness during the compression. A necessary criterion for equivalent transformation is that the resulting network is as expressive as the original one. Methods to study neural network expressiveness include universal approximation theory~\cite{Cybenko1989}, VC dimension~\cite{Bartlett1998}, trajectory length~\cite{raghu2017expressive}, and linear regions~\cite{pascanu2013on,montufar2014on,montufar2017notes,raghu2017expressive,arora2018understanding,serra2018bounding,hanin2019complexity,hanin2019deep,serra2020empirical}.
        
        We can also consider our approach as a form of reparameterization, or equivalent transformation, in graphical models~\cite{Koval1976,Wainwright2004,Werner2005}. If two parameter vectors $\theta$ and $\theta'$ define the same energy function (i.e., $E(\vx|\theta)) = E(\vx|\theta'),\forall~\vx$), then $\theta'$ is called a  reparameterization of $\theta$. Reparameterization has played a key role in several inference problems such as belief propagation~\cite{Wainwright2004}, tree-weighted message passing~\cite{Wainwright2005}, and graph cuts~\cite{Kolmogorov2007}. The idea is also associated with characterizing the functions that can be represented by \dnns~\cite{Hornik1989,Cybenko1989,Telgarsky2016,Lin2018,arora2018understanding,FunctionApproximation,kumar2019equivalent}.
        
        Finally, our work complements the vast literature at the intersection of mathematical optimization and Machine Learning~(ML). General-purpose  methods have been applied to train ML models to optimality~\cite{ClassificationTrees,TrainingLP,BinarizedTraining,CAQL}. Conversely, ML models have been extensively applied in optimization~\cite{CombOptTour,OptimizationSurvey}. To mention a few lines of work,  ML has been used to find feasible solutions~\cite{FeasibilityCP,FeasibilityMIP} and predict good solutions~\cite{EndToEnd,KnapsackPrediction}; determine how to branch on~\cite{BranchingLearning1,BranchingLearning2,BranchingSurvey,BranchingLearning3,HeuristicTreeSearch} or add cuts~\cite{RLCut}, when to linearize~\cite{LinearizationLearning}, or when to decompose~\cite{DecompositionLearning} an optimization problem; how to adapt algorithms for each problem~\cite{AdaptBrown,AdaptUBC,AdaptPatterns,AdaptSurvey,LearnGraphAlgo,LearnTSP1,LearnTSP2}; obtain better optimization bounds~\cite{BetterBounds}; embed the structure of the problem as a layer of a neural  network~\cite{OptNet,DDGAN,agrawal2019differentiable,MIPaaL}; and  predict the resolution by a time-limit~\cite{ResolutionLearning}, the feasibility of the problem~\cite{FeasibilityPrediction}, and the 
        problem itself~\cite{SPO,ModelingML,DeepInverse}.

\section{\texttt{LEO}: Lossless Expressiveness Optimization}
    Algorithm~\ref{alg:compress}, which we denote \texttt{LEO}~(Lossless Expressiveness Optimization), 
    loops over the layers to remove units with constant outputs regardless of the input, some of the stable units, and any layers with constant output due to those two types of units. These modifications of the network architecture are followed by changes to the weights and biases of the remaining units in the network to preserve the outputs produced. The term \emph{expressiveness} is commonly used 
    to refer to the ability of a network architecture to represent complex functions~\cite{serra2020empirical}.  
    
    First, \texttt{LEO} checks the weights and stability of each unit and decides whether to immediately remove them. A unit with constant output, which is either stably inactive or has zero input weights, is removed as long as there are other units left in the layer. A stably active unit with varying output is removed if the column of weights of that unit is linearly dependent on the column of weights of stably active units with varying outputs that have been previously inspected in that layer. We consider the removal of such stably active units as a \emph{merging} operation since the output weights of other stable units need to be adjusted as well. 
    
    Second, \texttt{LEO} checks if layers can be removed in case the units left in the layer are all stable or have constant output. If they are all stably active with varying output, then the layer is removed by directly joining the layers before and after it, which we denote as a \emph{folding} operation. In the particular case that only one unit is left with constant output, be it stably inactive or not, then all hidden layers are removed because the network has an associated function that is constant in $\sD$. We denote the latter operation as \emph{collapsing} the neural network. 
    
    Figure~\ref{fig:example} shows examples of units being removed and merged on the left as well as of a layer being folded on the right. Although possible, folding or collapsing a trained neural network is not something that we would expect to achieve in practice unless we are compressing with respect to a very small domain $\sD \subset \sX$.
    
    \begin{figure}
        \begin{center}
            \includegraphics[width=\textwidth]{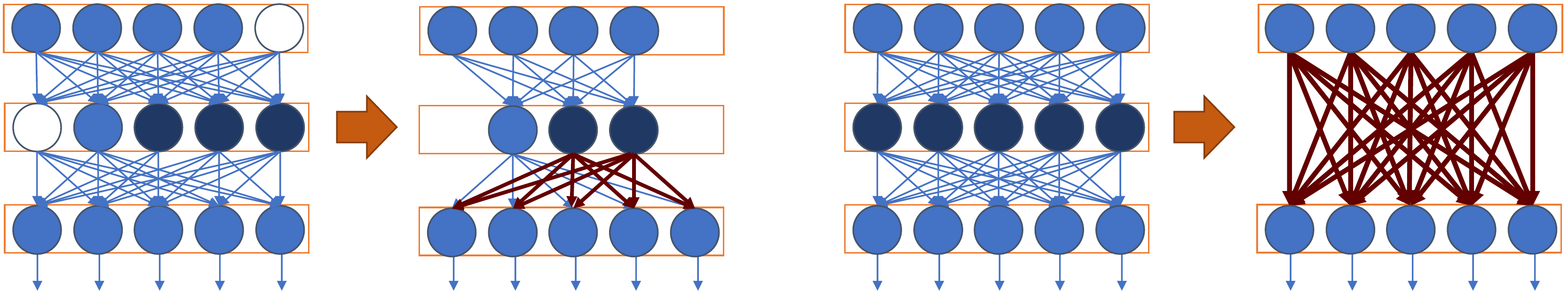}
            \caption{Examples of output-preserving neural network compression obtained with \texttt{LEO}. 
            On the left, two units in white are stably inactive and three units indexed by set $S$ in darker blue are stably active, where rank($\mW^2_S$)=2. In such a case, we can remove the stably inactive units and merge the stably active units to produce the same input to the next layer using only two units. 
            On the right, an entire layer is stably active. In such a case, we can fold the layer by directly connecting the layers before and after it. In both cases, the  red arcs correspond to coefficients that need to be adjusted accordingly. 
            } \label{fig:example}
        \end{center}
    \end{figure}

    \begin{algorithm}[!tb]
    \caption{\texttt{LEO} produces a smaller equivalent neural network with respect to a domain $\sD$ by removing units and layers while adjusting weights and biases} 
    \label{alg:compress}
    {\footnotesize
    \begin{algorithmic}[1]
    \For{$l \gets 1, \ldots, L$}
    \State $S \gets \{ \}$ \Comment{Set of stable units left in layer $l$}
    \State Unstable $\gets$ False \Comment{If there are unstable units in layer $l$}
    \For{$i \gets 1, \ldots, n_l$}
    \If{$\mathcal{G}_i^l < 0$ for $x \in \sD$ or $\mW_i^l=$\textbf{0}} \label{lin:inactive}
    \Comment{Unit $i$ has constant output}
    \If{$i < n_l$ \textbf{or} $|S| > 0$ \textbf{or} Unstable} \label{lin:one_inactive}
    \Comment{Layer $l$ still has other units} 
    \If{$\mW_i^l=$ \textbf{0} and $b_i^l > 0$}
    \For{$j \gets 1, \ldots, n_{l+1}$} \Comment{Adjust activations in layer $l+1$} 
    \State $b^{l+1}_j \gets b^{l+1}_j + w^{l+1}_{ji} b^l_i$
    \EndFor
    \EndIf
    \State Remove unit $i$ from layer $l$
    \Comment{Unit $i$ is not necessary}
    \EndIf
    \ElsIf{$\bar{\mathcal{G}}_i^l > 0$ for $x \in \sD$} \label{lin:active}
    \Comment{Unit $i$ is stably active}
    \If{rank$\left(\mW^l_{S \cup \{ i \}}\right) > |S|$} \Comment{$w^l_i$ is linearly independent} \label{lin:li_active} 
    \State $S \gets S \cup \{ i \}$ \Comment{Keep unit in the network}
    \Else \Comment{Output of unit $i$ is linearly dependent} \label{lin:ld_active}
    \State Find $\{ \alpha_k \}_{k \in S}$ such that $w^l_i = \sum_{k \in S} \alpha_k w^l_k$
    \For{$j \gets 1, \ldots, n_{l+1}$} \Comment{Adjust activations in layer $l+1$} 
    %
    \State $w^{l+1}_{jk} \gets w^{l+1}_{jk} + \sum_{k \in S} \alpha_k w^{l+1}_{ji}$
    \State $b^{l+1}_j \gets b^{l+1}_j + w^{l+1}_{ji} (b^l_i + \sum_{k \in S} \alpha_k  b^l_k)$
    \EndFor
    \State Remove unit $i$ from layer $l$ 
    \Comment{Unit $i$ is no longer necessary}
    \EndIf
    \Else \label{lin:unstable}
    \State Unstable $\gets$ True 
    \Comment{At least one unit is not stable}
    \EndIf
    \EndFor
    \If{\textbf{not} Unstable} \Comment{All units left in layer $l$ are stable}
    \If{$|S|>0$}  \label{lin:fold}
    \Comment{The units left have varying outputs}
    \State Create matrix $\bar{\mW} \in \mathbb{R}^{n_{l-1} \times n_{l+1}}$ and vector $\bar{\vb} \in \mathbb{R}^{n_{l+1}}$
    \For{$i \gets 1, \ldots, n_{l+1}$} \Comment{Directly connect layers $l-1$ and $l+1$} 
    %
    \State $\bar{b}_i \gets b^{l+1}_i + \sum_{k \in S} w^{l+1}_{ik} b^l_k$
    \For{$j \gets 1, \ldots, n_{l-1}$}
    \State $\bar{w}_{ij} \gets \sum_{k \in S} w^l_{kj} w^{l+1}_{ik}$
    \EndFor
    \EndFor
    \State Remove layer $l$; replace parameters in next layer with $\bar{\mW}$ and $\bar{\vb}$ 
    \Else \label{lin:collapse}
    \Comment{Only unit left in layer $l$ has constant output}
    \State Compute output $\Upsilon$ for any input $\chi \in \sD$ \Comment{Function is constant} 
    %
    \State $(\mW^{L+1},\vb^{L+1}) \gets$ $($\textbf{0}$, \Upsilon)$
    \Comment{Set constant values in output layer}
    \State Remove layers 1 to $L$ and \textbf{break} \Comment{Remove all hidden layers and leave}
    \EndIf
    \EndIf
    \EndFor
    \end{algorithmic}
    }
    \end{algorithm}

    \begin{theorem}\label{th:1}
    For a neural network \dnnone, Algorithm~\ref{alg:compress} produces a neural network \dnntwo~such that \dnnone~and \dnntwo~are local equivalent with respect to an input domain $\sD$. 
    \end{theorem}
    \begin{proof}
        If $\mathcal{G}_i^l < 0$, then $h^l_i = 0$ for any input in $\sD$ and unit $i$ in layer $l$ can be regarded as stably inactive. Otherwise, if $\mW_i^l =$ \textbf{0}, then the output of the unit is positive but constant. Those two types of units are analyzed by the block starting at line~\ref{lin:inactive}.  If there are other units left in the layer, which are either not stable or stably active but not removed, then removing unit a stably inactive unit $i$ does not affect the output of subsequent units since the output of the unit is always 0 in $\sD$. Likewise, in the case that $\mW_i^l =$ \textbf{0} and $b^l_i > 0$, then 
        the output of the network remains the same after removing that unit if such removal of $h_i^l$ from  each unit $j$ in layer $l+1$ is followed by adding
        $w^{l+1}_{ji} b^l_i$ to $b_j^{l+1}$.

        If $\bar{\mathcal{G}}_i^l > 0$, then $h_i^l = \mW^l_i h^{l-1} + b_i^l$ for any input in $\sD$ and unit $i$ in layer $l$ can be regarded as stably active. Those units are analyzed by the block starting at line~\ref{lin:active}. If the rank of the submatrix $\mW^l_S$ consisting of the weights of stably active units in set $S$ is the same as that of $\mW^l_{S \cup \{i\}}$ and given that $\mW_i^l \neq$ \textbf{0} for every $i \in S$, then $S \neq \emptyset$ and $h^l_i = \sum\limits_{k \in S} \alpha_k w_k^l h^{l-1} + b_i^l = \sum\limits_{k \in S} \alpha_k (h^l_k - b^l_k) + b_i^l$. Since there would be other units left in the layer, the output of the network remains the same after removing the unit if such removal of $h^l_i$ from  each unit $j$ in layer $l+1$ is followed by adding $\alpha_k w^{l+1}_{ji}$ to $w^{l+1}_{jk}$ and $w^{l+1}_{ji} \left( b_i^l - \sum\limits_{k \in A} \alpha_k b^l_k \right)$ to $b_j^{l+1}$. 
        
        If all units left in layer $l$ are stably active and $|S|>0$, 
        then layer $l$ is equivalent to an affine transformation and it is possible to 
        directly connect layers $l-1$ and $l+1$,   
        as in 
        the block starting at line~\ref{lin:fold}. 
        Since $h^l_k = W^l_j h^{l-1} + b^l_k$ for each stably active unit $k$ in layer $l$, 
        then $h^{l+1}_i = W^{l+1}_i h^l + b^{l+1}_i =  W^{l+1}_i \left( \sum\limits_{k=1}^{n_l} W^l_k h^{l-1} + b^l_k \right) + b^{l+1}_i = \sum\limits_{j \in n_{l-1}} \left( \sum\limits_{k \in S} w^l_{kj} w^{l+1}_{ik} \right) h^{l-1}_j + b^{l+1}_i + \left( \sum\limits_{k \in S} w^{l+1}_{ik} b^l_k \right)$.
        
        If the only unit $i$ left in layer $l$ is stably inactive or stably active but has zero weights, then any input in $\sD$ results in 
        $h_i^l = \max\{0,b_i^l\}$, 
        and consequently the neural network is associated with a constant function $f : x \rightarrow \Upsilon$ in $\sD$. Therefore, it is possible to remove all hidden layers and replace 
        the output layer with a constant function mapping to $\Upsilon$ as in the block starting at line~\ref{lin:collapse}. $\Box$
    \end{proof}
    
    \paragraph{Implementation} We do not need to solve (\ref{eq:obj})--(\ref{eq:last_mi}) to optimality to determine if $\mathcal{G}_i^l < 0$: it suffices to find a negative upper bound to guarantee that, or a solution with positive value to refute that. A similar reasoning applies to $\bar{\mathcal{G}}_i^l > 0$.


\pagebreak

\section{Experiments}
    We conducted experiments to evaluate the potential for network compression using \texttt{LEO}. 
    In these experiments, 
    we trained rectifier networks on the MNIST dataset~\cite{LeCun1998} with input size 784, two hidden layers of same width, and 10 outputs. 
    The widths of the hidden layers are 25, 50, and 100. 
    For each width, we identified in preliminary experiments a weight for $\ell_1$ regularization on layer weights that improved the network accuracy in comparison with no regularization: 0.0002, 0.0002, and 0.0001, respectively. 
    We trained 31 networks with that  regularization weight, with 5 times the same weight to induce more stability, and with zero weight as a benchmark. 
    We use the negative log-likelihood as the loss function after taking a softmax operation on the output layer, 
    a batch size of $64$ and SGD with a learning rate of $0.01$, and momentum of $0.9$ for training the model to $120$ epochs. The learning rate is decayed by a factor of $0.1$ after every $50$ epochs. The weights of the network were initialized with the Kaiming initialization~\cite{He2016DeepRL} and the biases were initialized to zero. The models were trained using  Pytorch~\cite{paszke2017automatic} 
    on a machine 
    with 40 Intel(R) Xeon(R) CPU E5-2640 v4 @ 2.40GHz processors and 132 GB of RAM.  
    The MILPs were solved 
    using Gurobi 8.1.1~\cite{gurobi} on a machine 
    with Intel(R) Core(TM) i5-6200U CPU @ 2.30 GHz and 16 GB of RAM. We used callbacks to check 
    bounds and solutions and then interrupt the solver after determining unit stability and bounds for each MILP. 

Tables~\ref{tab1} and \ref{tab2} summarize our experiments with mean and standard error. We note that the compression grows with the size of the network and the weight of $\ell_1$ regularization, which induces the weights of each unit to be orders of magnitude smaller than its bias. The  compression identified is all due to removing stably inactive units. Most of the runtime is due to solving MILPs for the second hidden layer. Given the incidence of stably active units, we conjecture that inducing rank deficiency in the weights or negative values in the biases could also be beneficial. 

\begin{table}
\caption{Additional details about the experiments for each type of network, including runtime per test, incidence of stably active units, and overall network stability.}
\label{tab2}
\centering
\begin{tabular}{@{\extracolsep{6pt}}cccccc}
&&& \multicolumn{2}{c}{Stably active units} &  Network \\
\cline{4-5}
Layer width & $\ell_1$ weight & Runtime~(s) &  $1^{\text{st}}$ layer & $2^{\text{nd}}$ layer &  stability~(\%) \\
\hline
25 & 0.001 & 27.9 $\pm$ 0.3 & 2.5 $\pm$ 0.3 & 7.4 $\pm$ 0.4 & 41.3 $\pm$ 0.6 \\ 
25 & 0.0002 & 29 $\pm$ 1 & 0 $\pm$ 0 & 1.0 $\pm$ 0.2 & 10.4 $\pm$ 0.7 \\ 
25 & 0 & 28.4 $\pm$ 0.3 & 0 $\pm$ 0 & 0 $\pm$ 0 & 0 $\pm$ 0 \\ 
\hline
50 & 0.001 & 103 $\pm$ 2 & 15 $\pm$ 0.5 & 24.9 $\pm$ 0.6 & 69.3 $\pm$ 0.4 \\ 
50 & 0.0002 & 106 $\pm$ 3 & 2.7 $\pm$ 0.3 & 8.8 $\pm$ 0.5 & 26.6 $\pm$ 0.6 \\ 
50 & 0 & 112 $\pm$ 3 & 0 $\pm$ 0 & 0 $\pm$ 0 & 0 $\pm$ 0 \\ 
\hline
100 & 0.0005 & 421 $\pm$ 4 & 35.7 $\pm$ 0.6 & 57.7 $\pm$ 0.7 & 77.5 $\pm$ 0.2 \\ 
100 & 0.0001 & 456 $\pm$ 8 & 11.1 $\pm$ 0.5 & 18 $\pm$ 0.7 & 29.4 $\pm$ 0.5 \\ 
100 & 0 & 385 $\pm$ 2 & 0 $\pm$ 0 & 0 $\pm$ 0 & 0 $\pm$ 0 \\ 
\end{tabular}
\end{table}

\section{Conclusion}\label{sec:conclusion}

We introduced a lossless neural network compression algorithm, \texttt{LEO}, which relies on MILP to identify 
parts of the neural network that can be safely removed after reparameterization. 
We found that networks trained with $\ell_1$ regularization are particularly amenable to such compression. In a sense, we could interpret $\ell_1$ regularization as inducing a subnetwork to represent the function associated with the DNN. 
Future work may explore the connection between these subnetworks identified by \texttt{LEO} and lottery tickets, 
bounding techniques such as in~\cite{WongKolter} to help efficiently identifying stable units, and 
other forms of inducing posterior compressibility while training.
Concomitantly, we have shown another form in which discrete optimization can play a key role in deep learning applications.

\bibliographystyle{splncs04}
\bibliography{reference}

\end{document}